\documentclass{article}

\usepackage{microtype}
\usepackage{graphicx}
\usepackage{subfigure}
\usepackage{booktabs} 

\usepackage{hyperref}

\usepackage[accepted]{icml2023}

\usepackage{amsmath}
\usepackage{amssymb}
\usepackage{mathtools}
\usepackage{amsthm}

\usepackage[capitalize,noabbrev]{cleveref}

\theoremstyle{plain}
\newtheorem{theorem}{Theorem}[section]

\theoremstyle{definition}

\theoremstyle{remark}

\usepackage[textsize=tiny]{todonotes}

\icmltitlerunning{SeMAIL: Eliminating Distractors in Visual Imitation via Separated Models}

\begin{document}

\twocolumn[
\icmltitle{SeMAIL: Eliminating Distractors in Visual Imitation via Separated Models}



\icmlsetsymbol{equal}{*}

\begin{icmlauthorlist}
\icmlauthor{Shenghua Wan}{nju}
\icmlauthor{Yucen Wang}{nju}
\icmlauthor{Minghao Shao}{nju}
\icmlauthor{Ruying Chen}{nju}
\icmlauthor{De-Chuan Zhan}{nju}
\end{icmlauthorlist}

\icmlaffiliation{nju}{National Key Laboratory for Novel Software Technology
, Nanjing University, Nanjing, China}

\icmlcorrespondingauthor{De-Chuan Zhan}{zhandc@nju.edu.cn}

\icmlkeywords{Machine Learning, ICML}

\vskip 0.3in
]



\printAffiliationsAndNotice{}  

\begin{abstract}
Model-based imitation learning (MBIL) is a popular reinforcement learning method that improves sample efficiency on high-dimension input sources, such as images and videos. Following the convention of MBIL research, existing algorithms are highly deceptive by task-irrelevant information, especially moving distractors in videos. To tackle this problem, we propose a new algorithm - named Separated Model-based Adversarial Imitation Learning (SeMAIL) - decoupling the environment dynamics into two parts by task-relevant dependency, which is determined by agent actions, and training separately. In this way, the agent can imagine its trajectories and imitate the expert behavior efficiently in task-relevant state space. Our method achieves near-expert performance on various visual control tasks with complex observations and the more challenging tasks with different backgrounds from expert observations.

\end{abstract}

\section{Introduction}
\label{introduction}
Reinforcement learning enables agents to autonomously learn specific behaviors and acquire diverse skills by interacting in environments. However, the rewards used for agent training are hard to define due to the expensive professional knowledge in domains \cite{DBLP:conf/nips/Hadfield-Menell17,everitt2021reward}. One promising direction is imitation learning \cite{DBLP:journals/jmlr/RossGB11,DBLP:journals/ftrob/OsaPNBA018}, where agents learn desired behaviors directly from expert demonstrations, especially from more readily available data such as images and videos \cite{DBLP:conf/iclr/StadieAS17,DBLP:conf/corl/YoungGT0AP20}. With model-based methods introduced in this field \cite{baram2016model,DBLP:conf/icml/FuYAJ21,rafailov2021visual,mile2022}, agents can sample data and train policy in the learned model, which significantly improves the sample efficiency and representation learning ability in visual observations.
In model-based visual imitation learning algorithms, there is a common issue caused by distractors that majorly hinder learning ability in complex real-world scenarios, as shown in \cref{fig:motivation}. In this case, the agent already obtained task-completion information from the expert by exploring specific actions. However, the environment model contains various task-irrelevant details and will take that irrelevant information into the next predicted state. Because of the difference in task-irrelevant parts, the agent only can receive a low reward from the discriminator even if the action in the task-relevant part nicely dovetails with the target. 

\begin{figure}[t]
\vskip 0.1in
\begin{center}
\centerline{\includegraphics[width=\columnwidth]{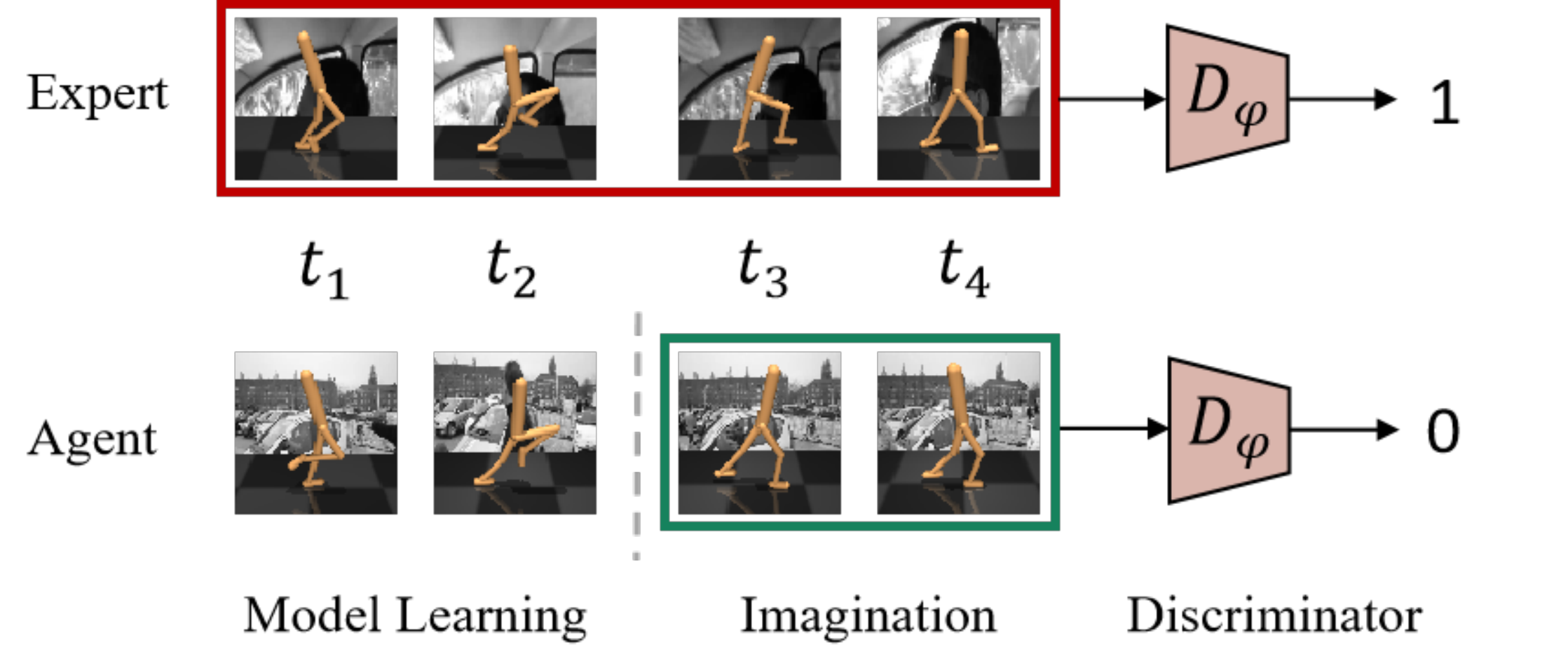}}
\caption{Illustration of convention model-based imitation learning which trains the policy in a single environment model. Real-world cases may contain a large number of distractors, especially moving distractors like walkers walking on the roadside. The model trained in these cases will tell the agent not only the result of its action but also something else irrelevant to the task. The discriminator may mistakenly use those imagined distractors to distinguish the agent and expert, even if they have similar task-relevant behaviors.}
\label{fig:motivation}
\end{center}
\vskip -0.1in
\end{figure}

Recent studies attempted to exclude distractors by bottlenecking the information between latent states and observations \cite{DBLP:conf/iclr/PengKTAL19} or extracting features unrelated to the task from the initial frames explicitly \cite{DBLP:conf/corl/ZolnaR0CBCDF020}. Other works match the policy and expert by minimizing mutual information using randomly pre-collected data in both domains \cite{DBLP:conf/iclr/StadieAS17, DBLP:conf/iclr/CetinC21}. In real-world practice, these methods still have shortcomings : (1) Only address IL problems on low-dimensional inputs or images with static distractors. (2) Pre-collected data is required for aligning the agent and the expert. To close the gap, we focus on excluding task-irrelevant information without additional requirements, which benefits solving real-world cases.

When the agent executes the same action based on the same states, the task-relevant parts will have a similar transition; contrariwise, the task-irrelevant parts will have independent transitions. We claim that only the task-relevant states in observation will change consistently when the agent and expert make a specific action. Based on this finding, we make the Action-conditioned Transition (AcT) assumption. We propose a practical method named Separated Model-based Adversarial Imitation Learning (SeMAIL) based on the assumption. For model learning, we design two models to estimate the forward dynamics for these two parts and dissociate the irrelevant one from the action input. The two separated models are trained on data from both expert and agent buffer. After that, We enforce task-relevant and irrelevant latent states together to reconstruct the original observation to learn the corresponding representations. We train the agent by adversarial imitation learning for policy learning, but the policy data used in discriminator and policy training is collected only in the task-relevant latent state space.

To begin with, we evaluate our method on modified DeepMind Control Suite tasks \cite{DBLP:journals/corr/abs-1801-00690}, which have complex real-scene video distractors in the background, and a 2-D classical control task from OpenAI Gym \cite{brockman2016openai}. The results show that SeMAIL achieves near-expert performance on various tasks. Further, SeMAIL and its variants outperform compared methods on a series of tasks with different distractors from expert demonstrations. From those sides, our proposed method can capture the most helpful information (e.g. agent joints or target body) for completing tasks, especially for tasks with complex and time-correlated distractors.

To summarize our contributions in this work: (i) We propose a new approach, Separated Model-based Adversarial Imitation Learning (SeMAIL), which aims to solve IL tasks from visual observations with complex or time-correlated distractors. (ii) We provide a theoretical analysis that shows the performance gap between agent and expert in observation space can also be bounded in the task-relevant latent state space. (iii) We show the superior performance of policies learned by SeMAIL across multiple visual control tasks with complex distractors, including learning from expert demonstrations with backgrounds different from observations.

\section{Related Work}

\textbf{Reward Estimation in Visual Imitation Learning.}
A practical way for imitation learning is to get the expert's supervised signal directly from high-dimensional inputs like images or videos \cite{pmlr-v78-finn17a, pathakICLR18zeroshot, DBLP:conf/corl/YoungGT0AP20,DBLP:journals/corr/abs-2206-15469}. Matching the agent and expert distribution based on extracting representations from visual observations and estimating the correct reward is an inherently challenging task. Previous research generates rewards for agent training based on aligned expert and agent trajectories through proprioceptive information  \cite{DBLP:conf/ijcai/TorabiWS19}. P-SIL \cite{cohen2021imitation} uses the Sinkhorn distance \cite{DBLP:conf/nips/Cuturi13} to estimate the reward to guide the agent learning. Recent work \cite{liu2022visual} proposes a new approach that recovers the reward based on the similarity of the image patch between the agent and expert observations, but it can not estimate the reward precisely if the observation contains much task-irrelevant information. Our method aims to provide task-relevant-based rewards so that it can solve visual imitation learning tasks with complex and time-correlated distractors.

\textbf{Model-based Imitation Learning.} Several works have introduced model-based approaches \cite{sutton1990time, DBLP:conf/icml/DeisenrothR11} to imitation learning to address the problem of sampling efficiency. MAIL \cite{baram2016model} trains the model and the discriminator on off-policy data for adversarial imitation learning, but it only focuses on solving low-dimensional state tasks. GCL \cite{DBLP:conf/icml/FinnLA16} equips inverse reinforcement learning with a local linear dynamics model to learn a good cost function but can not directly learn from natural images. For visual inputs, V-MAIL \cite{rafailov2021visual} is a representative model-based imitation learning method that learns a surrogate model for underlying MDPs and provides a theoretical analysis of the performance bound between the agent and the expert on partially observed MDPs. MILE \cite{mile2022} trains a recurrent state-space model on an offline set of expert data without any interaction with the environment and can handle the image observation with a large size. All of these methods learn an environment model to approximate the underlying dynamics to improve sample efficiency in imitation learning. However, they ignore the possible distractors with their own transition dynamics, common in real-world cases, and can not prioritize task-relevant features in the learning process. If these distractors are absorbed into latent states, it will not only influence the next stage in the model learning process but further make the agent and discriminator focus on some task-irrelevant features.

\textbf{Reinforcement Learning with Noisy Observations.}\label{sec:rl with noise} Many recent RL studies have explored ways for better performance in environments with noise or distractors. These methods can be divided into two categories: reconstruction-free and reconstruction-based. As reconstruction-free methods, DBC \cite{DBLP:conf/iclr/0001MCGL21} learns a compact latent state by bisimulation metric to filter out distractors in the environment. InfoPower \cite{DBLP:conf/iclr/BharadhwajBEL22} combines a variational empowerment term into the state-space model to capture task-relevant features at first. These works substitute the reconstruction's functionality with other designs. For reconstruction-based methods, Denoised MDPs \cite{DBLP:conf/icml/0001D0IZT22} decomposites the visual observation into four parts by action and reward, and constructs the corresponding models. The most similar method to our approach is TIA \cite{DBLP:conf/icml/FuYAJ21}, which also designs two models to capture the task and distractor features. All these methods, including TIA, separate the task-related and distractor information through \textbf{rewards} based on human-designed prior knowledge. Our method has essential differences from TIA in model learning and policy learning. For model learning, we design a model to model task-relevant and irrelevant dynamics through conditioned actions instead of using ground-truth rewards as in TIA. For policy learning, we use a discriminator to train the agent by expert demonstrations in the task-relevant state space, while TIA trains it using rewards decoded by the learned model.

\section{Preliminaries}

A partially observed Markov decision process (POMDP) is an MDP in which agents must make decisions based on incomplete information \cite{kaelbling:aij98}. POMDP can be formalized as a 7-tuple $\langle \mathcal{S},\mathcal{A}, r, p, \mathcal{O},\Omega, \gamma \rangle$ where $\mathcal{S}$ is the state space, $\mathcal{A}$ is the action space, $r:\mathcal{S}\times\mathcal{A}\times\mathcal{S}\rightarrow \mathbb{R}$ is the reward function, $\mathcal{O}$ is the observation space, $p:\mathcal{S}\times \mathcal{A}\times \mathcal{S}\rightarrow [0, 1]$ is the state transition probability function, $\Omega:\mathcal{S}\times\mathcal{A}\times\mathcal{O}\rightarrow[0, 1]$ is the observation probability function and $\gamma\in[0, 1)$ is the discount factor. In a POMDP, the agent infers the belief state with incomplete observations of the environment and uses this information to make decisions that maximize its cumulative reward. Since the agent in imitation learning with complex distractors has no access to the ground-truth state, we can regard it as a type of POMDP\footnote{More discussions about the POMDP assumption are in \cref{app:BMDP}.}. The agent can only infer the belief state from the historical observations and accept the supervised signal from demonstrations collected by the optimal expert's policy $\pi^E$ in an inverse RL fashion.

Adversarial Imitation Learning (AIL) \cite{ho2016generative,DBLP:conf/iclr/FuLL18} is a typical sort of inverse RL algorithm, which takes the advantage of Generative Adversarial Networks (GAN)~\cite{goodfellow2014generative} to train a discriminator $D_{\psi}$ to distinguish between the agent trajectory collected from the environment and the expert data from fixed demonstrations. AIL aims to learn a policy $\pi$ to minimize the divergence between the expert and agent occupancy measures, meanwhile maximizing its entropy. It can be formulated as $\arg\min_{\pi}-H(\pi)+\psi^{\star}(\rho_{\pi}-\rho_{\pi_E})$, where $\psi$ is a regularizer and $\psi^{\star}$ is convex conjugate of it. There are various choices for $\psi$ \cite{DBLP:conf/corl/GhasemipourZG19}, and we adopt the JS divergence used in GAIL \cite{ho2016generative}. The objective is as follows: 
\begin{equation}
\begin{split}
    \max_{\pi}\min_{D_{\psi}} \;&\mathbb{E}_{(s, a)\sim \rho^E_{\mathcal{M}}}\Big[-\log D_{\psi}(s, a) \Big]\\
   & + \mathbb{E}_{(s, a)\sim \rho^{\pi}_{\mathcal{M}}}\Big[-\log(1-D_{\psi}(s, a))\Big]
\end{split}
\end{equation}

\begin{figure*}[t]
\vskip 0.1in
\begin{center}
\centerline{
\subfigure[Model learning]{\includegraphics[width=0.55\linewidth]{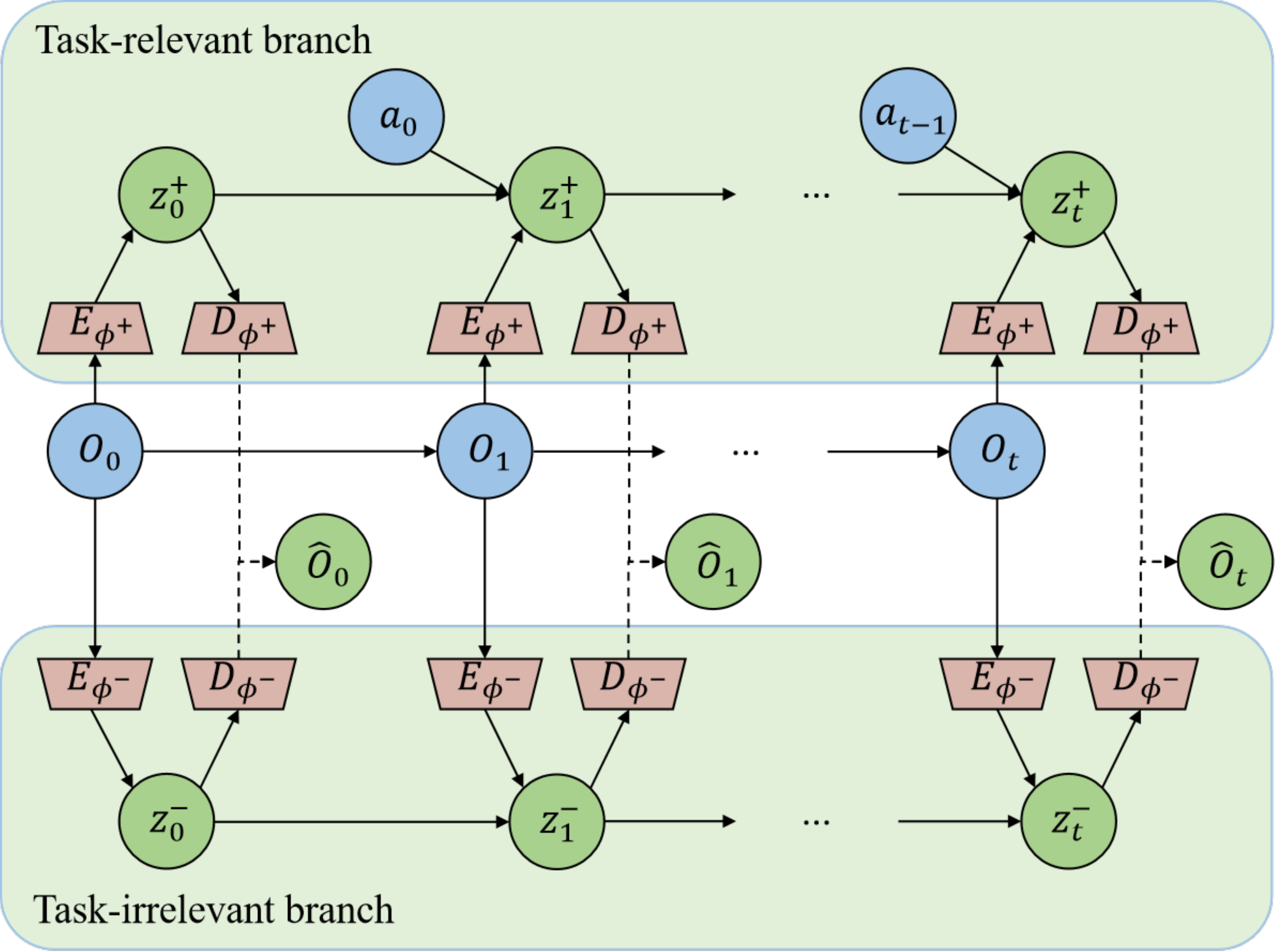}
\label{fig:model learning}}
\subfigure[Adversarial imitation learning in imagination]{\includegraphics[width=0.44\linewidth]{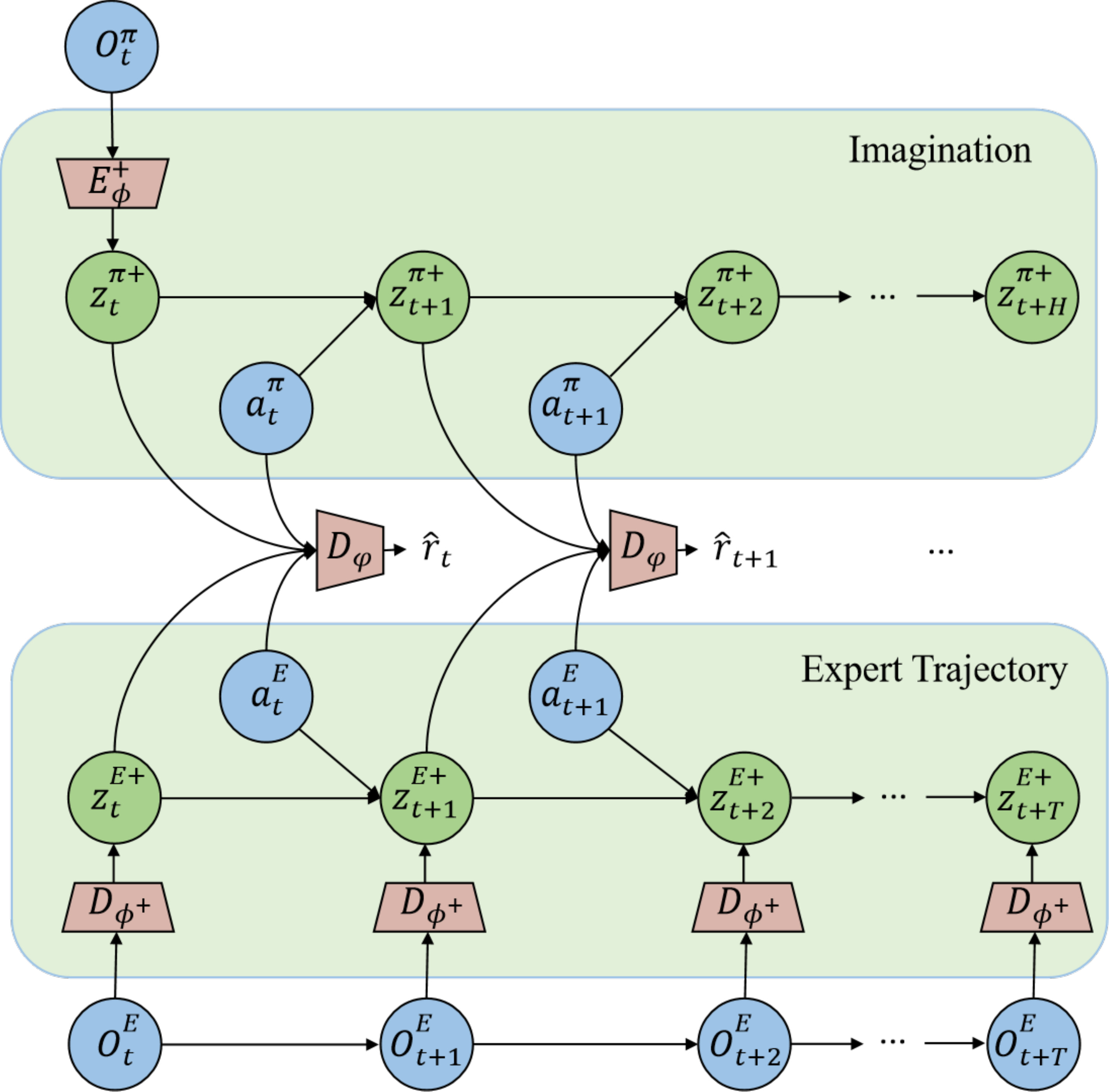}
\label{fig:policy learning}
}}
\vspace{-2pt}
\caption{Overview of SeMAIL. (a) The model learning consists of two branches: the top branch learns $p(z^+_{t}|z^+_{t-1}, a_{t-1})$ for task-relevant transition, and the bottom branch learns $p(z^-_{t}|z^-_{t-1})$ for distractor transition. $z^+_t$ and $z^-_t$ cooperatively reconstruct the observation $\hat{o}_t$ by decoders $D_{\phi^+}$ and $D_{\phi^-}$. (b) For policy learning, the agent samples trajectories in the learned task-relevant latent state space and updates its policy using rewards provided by the discriminator $D_{\psi}$.}
\label{fig:model_MBRL}
\end{center}
\vskip -0.1in
\end{figure*}

\section{Approach}
In this section, we first introduce the basic assumptions and overall architecture of SeMAIL (Section \ref{rssm2}). SeMAIL separates the model learning into two branches: task-relevant, which captures the task information with action inputs, and task-irrelevant, which learns the environmental background and other distractors dynamics without action inputs. We derive the lower bound of the mutual information between the observations and the latent states for observation reconstruction. We encourage the two models to reconstruct the original observations cooperatively to maximize the mutual information (Section \ref{observation reconstruction}). For policy learning, we use the GAIL framework to optimize the policy based on the theoretical result of \cref{thm:gail bound} (Section \ref{policy learning}).

\subsection{Learning Separated Models}
\label{rssm2}

We make a basic assumption, named Action-conditioned Transition (AcT), of the underlying mechanism of the environment, whose latent state $z_t$ consists of the task-relevant part $z_t^+$ and irrelevant part $z_t^-$ at time $t$. We assume that $z_t^+$ depends on the last task-relevant part $z_{t-1}^+$ and action $a_{t-1}$, and $z_t^-$ only depends on the last task-irrelevant part $z_{t-1}^-$. These two components $z^+_t$ and $z^-_t$ process \textbf{independently}\footnote{More discussions about the practice of AcT assumption can be found in \cref{app:AcT}.}. Thus, we can decompose the forward dynamics into two independent transition functions such that 
\begin{equation}
    \label{eq:AcT}
    p(z_t|z_{t-1}, a_{t-1}) = p(z_t^+|z_{t-1}^+, a_{t-1})p(z_{t}^-|z_{t-1}^-)
\end{equation}
The overall architecture is shown in \cref{fig:model learning}. First, we use two separated encoders $E_{\phi^+}$ and $E_{\phi^-}$ to extract the low-dimensional embedding $z_t^+$ and $z_t^-$ from original image observations $o_t$. Then, we design two forward transition models for learning the task and environment background dynamics separately: the \textit{task model} $p_{\theta^+}(z^{+}_{t+1}|z^{+}_t,a_t)$ which predicts the next task-relevant latent state, and the \textit{background model} $p_{\theta^-}(z^-_{t+1}|z^-_t)$ which infers the next irrelevant latent state. We simultaneously learn two variational encoders $q_{\psi^+}(z^+_t|o_t, z^+_{t-1}, a_{t-1})$ and $q_{\psi^-}(z^-_t|o_t, z^-_{t-1})$ to infer the posterior estimation of task-relevant and irrelevant latent states, respectively. To obtain compact representations of these two latent states, we jointly optimize the loss function in \cref{eq:rssm2} to minimize the KL divergence between the prior and posterior estimation of $z^+_t$ and $z^-_t$. We train the separated models on the data both from expert and agent buffers. The derivation of the loss function below is in \cref{app:deriv}.
\begin{equation}
\label{eq:rssm2}
\begin{split}
    \mathcal{L}_{\mathcal{\widetilde{M}}} &=~\mathbb{E}_{(o_{\tau},a_{\tau})\sim\mathcal{B}_{\pi}\cup\mathcal{B}_E}\\
    &\Bigg[\sum_{t=1}^T\mathbb{E}_{q(z_{t-1}^+|o_{1:t-1},a_{1:t-2})q(z_{t-1}^-|o_{1:t-1})}\\
    &\bigg(\mathbb{D}_{\text{KL}}\left[q_{\psi^-}(z^-_t|o_t, z^-_{t-1})||p_{\theta^-}(z^-_t|z^-_{t-1})\right]+\\
    &\mathbb{D}_{\text{KL}}\left[q_{\psi^+}(z^+_t|o_t, z^+_{t-1}, a_{t-1})||p_{\theta^+}(z^+_t|z^+_{t-1}, a_{t-1})\right]\bigg)\Bigg]
\end{split}
\end{equation}

\subsection{Joint Observation Reconstruction}
\label{observation reconstruction}
Most of previous model-based reinforcement learning methods consider auxiliary reconstruction-loss to optimize the observation encoder \cite{hafner2019dream,rafailov2021visual}. They formulate this loss as maximizing the mutual information $I(o_t;z_t)$ between the observations and latent states. We can optimize it by maximizing its BA lower bound \cite{agakov2004algorithm}:
\begin{equation}
    I(o_t;z_t)\ge \mathbb{E}_{p(o_t, z_t)}[\ln q_{\phi}(o_t|z_t)]+\mathcal{H}(p(o_t))\nonumber
\end{equation}
Because the observations do not depend on the latent state, optimization only considers maximizing the first term. However, enforcing the decoder to recover the whole observations from the encoded latent states will introduce bias when observations contain distractors. For imitation learning, it is catastrophic if the discriminator focuses on some task-irrelevant information and results in providing misleading reward to the agent. To avoid the latent state trained from all the details in observation, we maximize the mutual information between the observation $o_t$ and the combination of the task-relevant latent state $z_t^+$ and the irrelevant latent state $z_t^-$ as follows:
\begin{equation}
    I(o_t;z_t^+, z_t^-)\ge \mathbb{E}_{p(o_t, z_t^+, z_t^-)}[\ln q_{\phi}(o_t|z_t^+, z_t^-)]\nonumber
\end{equation}
To implement it, we design two decoders $D_{\phi^+}$ and $D_{\phi^-}$. With the given learned representation $z_t^+$ and $z_t^-$, decoder $D_{\phi^+}$ and $D_{\phi^-}$ will output the task-relevant visual component $\hat{o}^+_t$ with mask $M^+_t$ and the irrelevant visual component $\hat{o}_t^-$ with mask $M^-_t$. We use a 2D-convolutional layer to generate the final image mask $M_t$ from these two masks, a technique commonly utilized in prior works \cite{DBLP:conf/iccv/HeGDG17,DBLP:conf/icml/FuYAJ21}. These two visual components cooperatively recover the original observation at time $t$, formulated as $\hat{o}_t=M_t\odot \hat{o}_t^++(1 - M_t)\odot \hat{o}^-_t$.

Reconstructing the original image observation from two latent states may result in $z^+_t$ dominating the whole reconstruction process and thus containing some irrelevant information at time $t$. We assume that the task-relevant information is only a small proportion of the observation, which is also posed in \cite{DBLP:conf/icml/FuYAJ21,DBLP:conf/icml/0001D0IZT22}. To avoid the learning collapse of $z_t^-$ and too much information captured by $z_t^+$, we design an additional observation decoder $q_{\tilde{\phi}}$ following TIA to decode the whole observation. This term is designated as background-only reconstruction (BoR), with the aim that the non-controllable latent state $z_t^-$ can recover the task-irrelevant background as much as possible. The reconstruction loss\footnote{The derivation of the loss function is in \cref{app:deriv}} can be written as:
\begin{equation}\label{eq:obz_rec}
\begin{split}
\mathcal{L}_{O}=~&\mathbb{E}_{(o_{\tau},a_{\tau})\sim\mathcal{B}_{\pi}\cup\mathcal{B}_E}\\
&\Bigg[\sum_{t=1}^T\bigg(\mathbb{E}_{q(z_{t}^+|o_{1:t},a_{1:t-1})\atop q(z_{t}^-|o_{1:t})}[\ln q_{\phi^{+}, \phi^{-}}(o_t|z^+_t, z^-_t)] \\
&+\lambda\mathbb{E}_{q(z_{t}^-|o_{1:t})}[\ln q_{\tilde{\phi}}(o_t|z^-_t)]\bigg)\Bigg]
\end{split}
\end{equation}
where $\lambda$ is a hyper-parameter that controls the weight of the BoR term. In contrast to TIA, the task-irrelevant hidden state $z_t^{-}$ in the expectation of BoR loss does not rely on actions, as postulated by the AcT assumption.

\begin{figure*}[t]
\vskip 0.1in
\begin{center}
\centerline{\includegraphics[width=\linewidth]{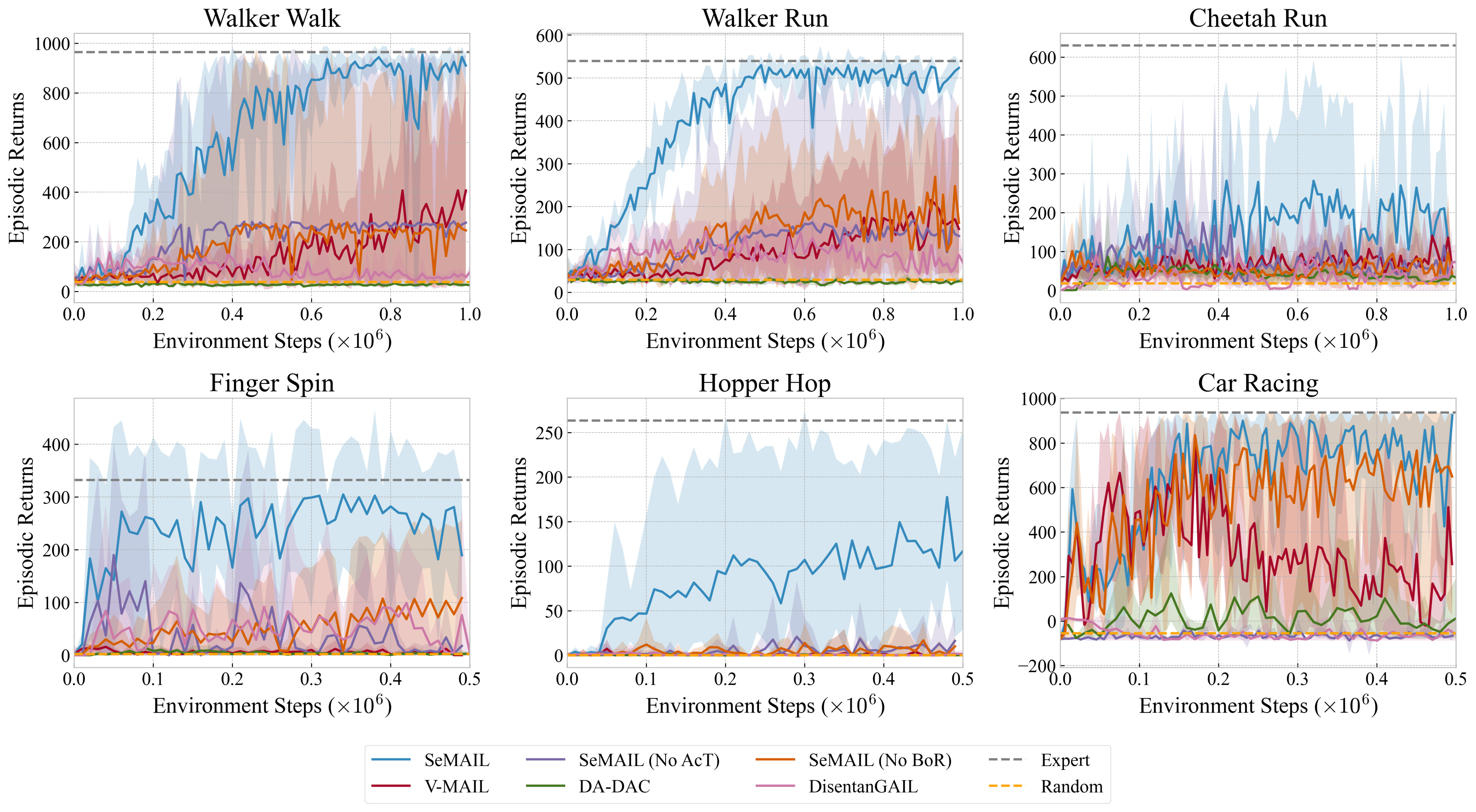}}
\caption{Evaluation results of our method SeMAIL and the baselines over four seeds in six visual control tasks. The solid curves present the average episodic returns, and the shaded region represents the range of performance under different runs. SeMAIL consistently outperforms the three compared methods in almost all environments.}
\label{fig:results}
\end{center}
\vskip -0.1in
\end{figure*}

\begin{figure*}[ht]
\vskip 0.1in
\begin{center}
\centerline{
\subfigure}{\includegraphics[width=0.17\linewidth]{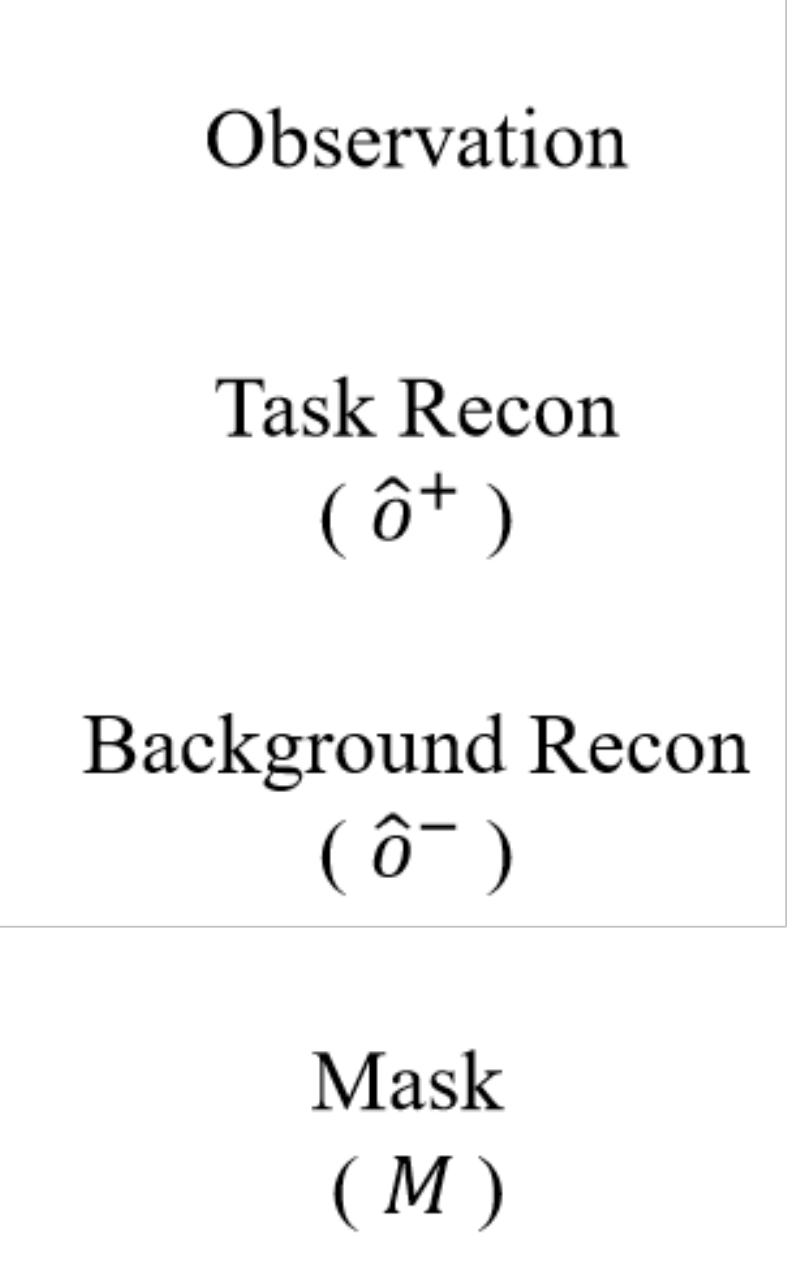}}
\subfigure[SeMAIL]{\includegraphics[width=0.27\linewidth]{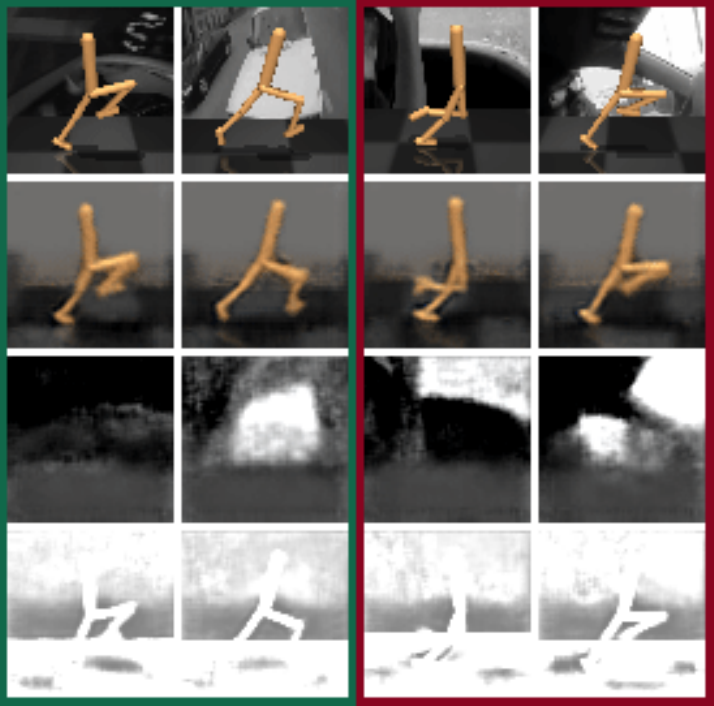}\label{fig:ab:SeMAIL}}
\subfigure[SeMAIL (No AcT)]{\includegraphics[width=0.27\linewidth]{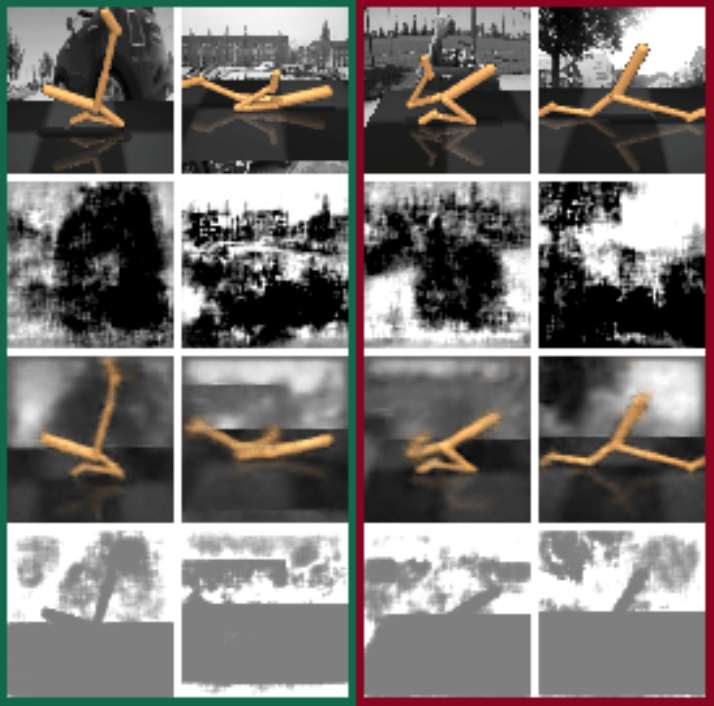}\label{fig:ab:no action}}
\subfigure[SeMAIL (No BoR)]{\includegraphics[width=0.27\linewidth]{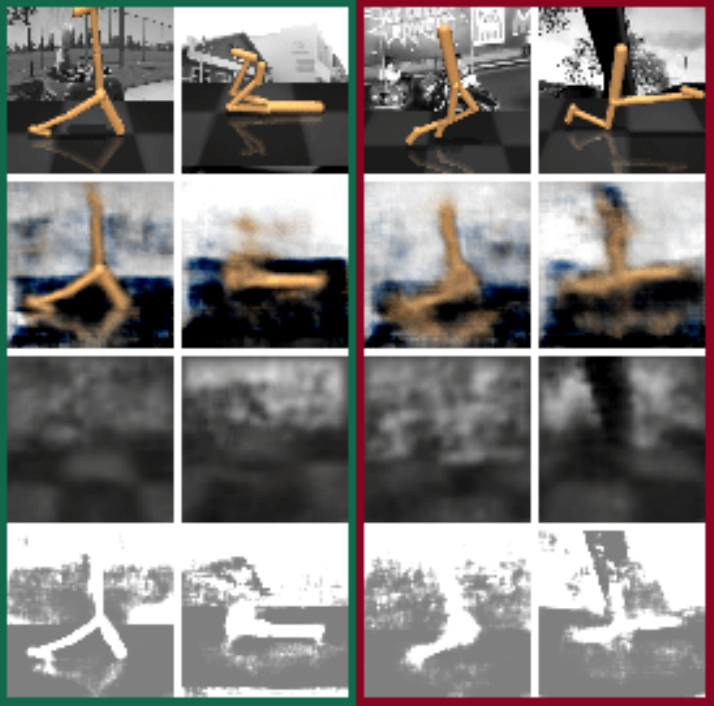}\label{fig:ab:no obs}}

\caption{Reconstruction results of SeMAIL and its ablated variants for different semantics in observations from the agent (green) and the expert (red). (a) SeMAIL can successfully separate the task-relevant part and irrelevant part from observations and reconstruct them. (b) SeMAIL (No AcT) reverses the reconstructions for these two parts. There is no task information in $\hat{o}^+$. (c) SeMAIL (No BoR) can reconstruct $\hat{o}^+$ and $\hat{o^-}$ well but contain some task-irrelevant information in the former.}
\label{fig:ablations}
\end{center}
\vskip -0.1in
\end{figure*}

\subsection{Policy Learning}
\label{policy learning}

In this section, we provide a theoretical upper bound on the performance gap between the expert and the policy trained on the task-relevant transition model. The agent can perform planning and adversarial imitation learning in task-relevant latent states and reach a similar performance to the expert theoretically. Based on the Lemma 1 stated in V-MAIL \cite{rafailov2021visual}, we extend the policy and model deviation from the original dynamics of MDPs to the task-relevant dynamics. Suppose there is a model $\widetilde{\mathcal{M}}^+$ which approximates the underlying task-relevant MDP $\mathcal{M}^+$ such that 
$\epsilon=\mathbb{D}_{\text{TV}}^{\max}(\widetilde{p}(s, a), p(s, a)), \forall (s, a)$.
Then, the expected returns gap between expert $\pi^E$ and learned policy $\pi$ can be bounded as 
\begin{equation}
\begin{split}
    &|J(\pi^E, \mathcal{M}^+)-J(\pi, \widetilde{\mathcal{M}}^+)|\le \\
    &\frac{R_{\max}}{1-\gamma}\mathbb{D}_{\text{TV}}(\rho^{\pi}_{\widetilde{\mathcal{M}}^+}, \rho^E_{\mathcal{M}^+})+\frac{\epsilon\cdot R_{\max}}{(1-\gamma)^2}
\end{split}
\end{equation}
where $R_{\max}=\max_{(s, a)}\mathcal{R}(s, a), \forall (s, a)$ is the maximum of the reward in the MDP with task-relevant dynamics.
We need to minimize the divergence $\mathbb{D}_{\text{TV}}(\rho^{\pi}_{\widetilde{\mathcal{M}}^+}, \rho^E_{\mathcal{M}^+})$ to achieve near-expert performance. We provide a theory that this divergence can also be bounded in the task-relevant latent state space $\mathcal{Z}^+$, which makes our proposed SeMAIL method available under the adversarial imitation learning framework. 
\begin{theorem}
\label{thm:gail bound}
Consider a POMDP $\mathcal{M}$ with high-dimensional inputs such as images. Let $s_t$ be the ground-truth state and $z_t$ be the latent representation of the whole observation $o_t$. $z_t^+$ and $z_t^-$ represent for the task-relevant and irrelevant component, respectively, which meets the condition of $p(s_t|z_t, a_t) = p(s_t|z_t^+, a_t)p(s_t|z_t^-)$ based on the AcT assumption mentioned in \cref{rssm2}. Then, we have
\begin{equation}
\begin{split}
    \mathbb{D}_f(\rho^{\pi}_\mathcal{M}(o, a)||\rho^{E}_\mathcal{M}(o, a)) \le &
     \mathbb{D}_f(\rho_{\mathcal{M}}^{\pi}(s, a)||\rho_{\mathcal{M}}^E(s, a))\\
     \le& \mathbb{D}_f(\rho_{\mathcal{M}}^{\pi}(z^+, a)||\rho_{\mathcal{M}}^E(z^+, a))
\end{split}
\end{equation}
where $\mathbb{D}_f$ is a generic $f$-divergence.
\end{theorem}

\cref{thm:gail bound} illustrates that the divergence of occupancy measures in the observation space can be upper-bounded in the learned task-relevant latent state space. The proof is in \ref{app:theorem}. As the learned model $\widetilde{\mathcal{M}}^+$ well approximates the true MDP $\mathcal{M}^+$, we can apply the adversarial imitation learning method in $\mathcal{Z}^+$ space to minimize the divergence $f(\rho_{\widetilde{\mathcal{M}}^+}^{\pi}(z^+, a)||\rho_{\mathcal{M}^+}^E(z^+, a))$ and reduce the performance gap between the agent and the expert. The objective is
\begin{equation}
\label{eq:discrim}
    \begin{split}
        \max_{\pi}\min_{D_{\psi}} \;&\mathbb{E}_{(z^+, a)\sim \rho^E_{\mathcal{M}}}\Big[-\log D_{\psi}(z^+, a) \Big]\\
   & + \mathbb{E}_{(z^+, a)\sim \rho^{\pi}_{\widetilde{\mathcal{M}}}}\Big[-\log(1-D_{\psi}(z^+, a))\Big]
    \end{split}
\end{equation}
The discriminator $D_{\psi}$ serves as an approximated reward function that gives an estimated reward $r^t$ of the agent's state-action pair at time $t$. The agent uses the pseudo reward to fit the value function bootstrapped and maximizes the expected return. To improve sample efficiency, we only train the discriminator on the fixed expert's demonstrations and the agent's rollouts in the learned model. The agent only samples actions based on the task-relevant state. 
\begin{equation}
\label{eq:policy}
    \begin{split}
        &\text{Action model:}\; a_t\sim \pi(a_t|z^+_t)\\
        &\text{Value model:}\; v(z_t^+) \approx \mathbb{E}_{\pi(\cdot|z_t^+)}\Big[\sum_{k=t}^{T}\gamma^{k-t}\log D_{\psi}(z_k^+, a_k)\Big]\nonumber
    \end{split}
\end{equation}
The policy learning process is presented in detail in \cref{fig:policy learning}.

\begin{figure*}[ht]
\vskip 0.1in
\begin{center}
\centerline{\includegraphics[width=\linewidth]{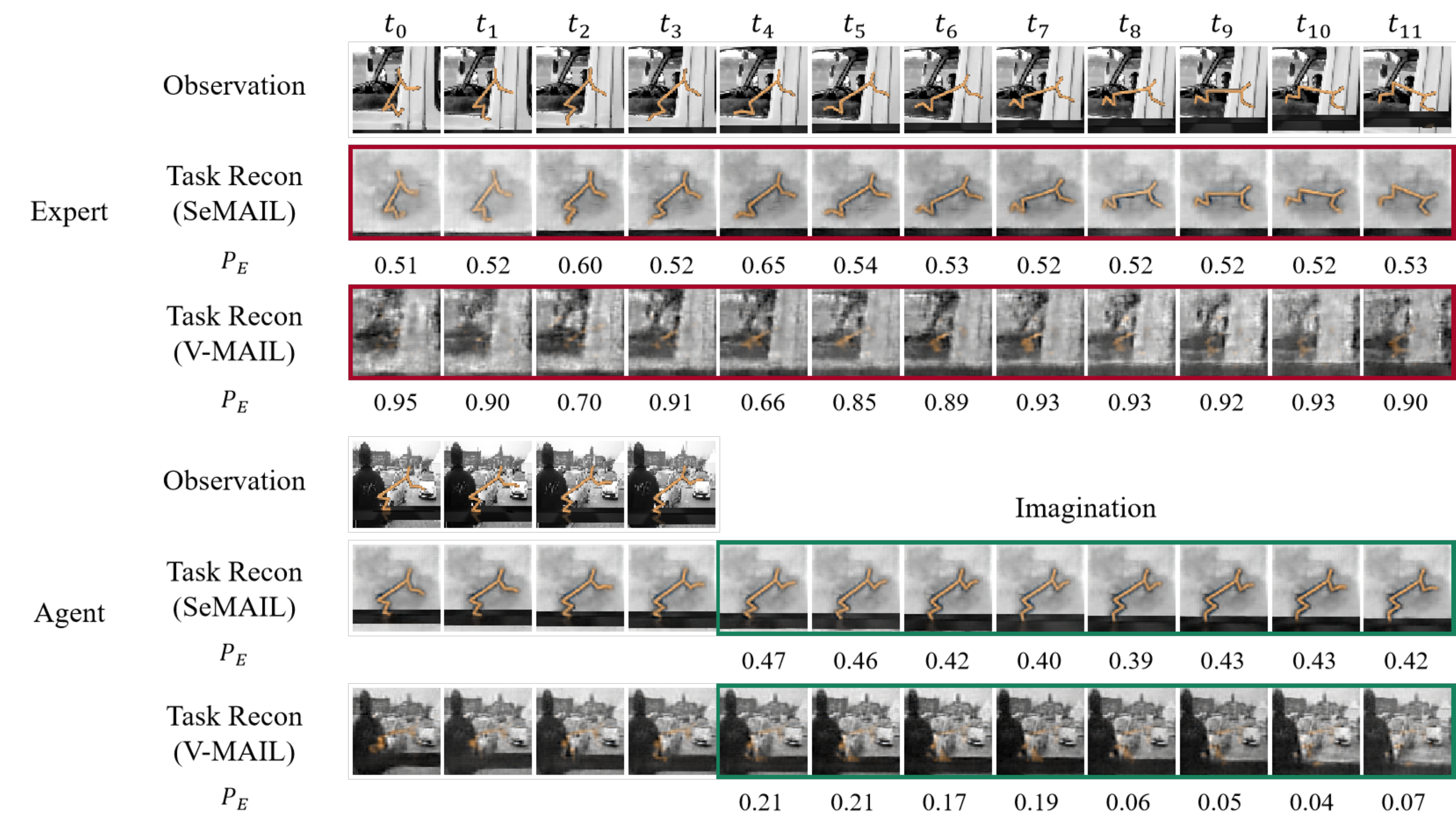}}
\caption{The adversarial imitation learning process of SeMAIL and V-MAIL. In latent space, both of SeMAIL and V-MAIL imagine the trajectory based on the observations of $t_0\sim t_3$ and imitate the expert behavior from all expert demonstrations. The reconstructions of the task model for the expert observation and the agent imagination are presented in red and green boxes, respectively. The value $P_E$ denotes the probability that discriminator $D_{\psi}$ predicts agent behaviors as expert behaviors, which is closer to 0.5 will be better.}
\label{fig:imagine}
\end{center}
\vskip -0.1in
\end{figure*}

\section{Experiments}

In our experiments part, we aim to answer the following questions:
\begin{enumerate}
\vspace{-0.15in} 
\item How good is SeMAIL's performance on learning tasks with complex distractors observations?
\item What are the roles played by the critical designs of SeMAIL in learning?
\item Could SeMAIL perform well with distractors that never appear in expert observations?
\item How well does the task-relevant space imagination benefit adversarial imitation learning?
\vspace{-0.15in}
\end{enumerate}

\textbf{Environments and Expert's demonstrations.}
We test our algorithm on six visual control tasks, \textit{i.e.}, five locomotion tasks from DeepMind Control (DMC) Suite \cite{DBLP:journals/corr/abs-1801-00690} with videos under the class        ``driving car'' of the Kinetics dataset \cite{DBLP:journals/corr/KayCSZHVVGBNSZ17} as background, and another Car Racing task from OpenAI Gym \cite{brockman2016openai}. These videos are grayscaled as in DBC \cite{DBLP:conf/iclr/0001MCGL21}. Instead of rendering the images from low-dimensional states or adding the distractors on the pure observations, we train RL policies from complex visual inputs until optimal as experts. Then we use the expert to collect demonstrations from the above environments directly. More details about environments and demonstrations are in \cref{app:details}.

\textbf{Baselines.}
 We design three versions of SeMAIL and compare them with baselines against three types of visual imitation learning approaches: model-based, data-augmented, and feature-space-shared. 
 \begin{itemize}
     \item \textbf{SeMAIL}: The full proposed method.
     \item \textbf{SeMAIL (No AcT)}: SeMAIL without the action-free constraint in the task-irrelevant model. The forward dynamics and posterior encoder of the task-irrelevant model are modified as $p_{\theta^-}(z^-_t|z^-_{t-1}, a_{t-1})$ and $q_{\psi^-}(z^-_t|o_t, z^-_{t-1}, a_{t-1})$, respectively.
     \item \textbf{SeMAIL (No BoR)}: SeMAIL without the background-only reconstruction loss.
     \item \textbf{V-MAIL}: The variational model-based adversarial imitation learning method \cite{rafailov2021visual}. 
     \item \textbf{DA-DAC}: The DrQ data augmentation \cite{DBLP:conf/iclr/YaratsKF21} version of Discriminator Actor Critic \cite{DBLP:conf/iclr/KostrikovADLT19} used in \cite{rafailov2021visual}.
     \item \textbf{DisentanGAIL}: The feature-space-shared IRL method that regularizes the latent representation with mutual information constraints \cite{DBLP:conf/iclr/CetinC21}.
 \end{itemize}

\subsection{How good is SeMAIL's performance on learning tasks with complex distractors observations?}
\label{exp: homo-source}
In \cref{fig:results}, we show the performance curves of all six visual control tasks. It is clear that SeMAIL outperforms all the baseline methods and enables high sampling efficiency in most of these tasks. SeMAIL achieves near-expert performance in Finger and Walker tasks while other methods struggle to solve them. On the Hopper Hop environment, SeMAIL has an outstanding average performance with a wider range of variance than the baselines. That is because the compared methods fail to learn the expert behavior, while for SeMAIL, there are more successful and fewer failed runs. In the Car Racing environment, V-MAIL obtains high episodic returns in the early stage, but the performance quickly drops to a value lower than SeMAIL as training goes on because the discriminator captures irrelevant parts in observations and distinguishes the agent from the expert.  

The performance on all six tasks has a considerable drop if we ignore the AcT assumption. Notably, the performances of SeMAIL (No AcT) become sharply unstable in Walker Walk and Walker Run. It indicates that dissociating task-irrelevant part from actions is crucial for task completion when some distractors have underlying dynamics. Removing BoR loss also leads to a significant performance drop on several tasks, except for Car Racing, since the distractors in this environment are not visually visible. Although removing AcT or BoR component, the average performances of SeMAIL still outperform the DA-DAC and DisentanGAIL in almost all environments. The quantitative result is in \cref{tb: homo-results}. 

\subsection{What are the roles played by the critical designs of SeMAIL in learning?}
\label{exp: ablation}

To further study the factors that cause the performance drop on SeMAIL (No AcT) and SeMAIL (No BoR), we visualize reconstructions of different semantics of observations from the agent and the expert in \cref{fig:ablations}. SeMAIL learns a distinct mask and reconstructs the observations well for both tasks and backgrounds. SeMAIL (No AcT) fails to learn an accurate mask and its two parts of the reconstruction are reversed, implying that the AcT assumption is crucial for extracting the corresponding features from $o^+$ and $o^-$. For SeMAIL (No BoR), the task model captures the task-relevant information in the correct semantics but also contains some background information, which may lead to poor performance results. Based on these ablation test results, the action-conditioned transition can help the model capture the corresponding representation with task-relevant information and majorly improves its learning ability. Background-only reconstruction tries to recover the whole background as much as possible via non-controllable latent state $z^{-}$ and can filter some irrelevant information that may be miscontained in controllable latent state $z^+$. We conclude that both of these two designs are absolutely necessary for imitation learning in environments with complex observations.

\subsection{Could SeMAIL perform well with distractors that never appear in expert observations?}
\label{exp: hete-source}
To test SeMAIL's ability by solving tasks with distractors that the expert has never seen before, we train the agent and expert in environments with non-overlap video backgrounds for Walker Run, Cheetah Run, and Finger Spin tasks. Due to the different backgrounds, the discriminator tends to distinguish the agent's observations from the expert's, which leads to low rewards and failure to learn expert behaviors. We record the mean and standard error of the maximum expected episodic returns over all the methods in tests and scale the values such that 0 represents the random agent performance and 1 represents the expert performance. The results in \cref{tb:hete-results} show V-MAIL and DA-DAC fail to solve the task with very low performances. DisentanGAIL performs well in the Finger Spin task but poorly in Walker and Cheetah tasks. It indicates that regularizing the mutual information between the agent and expert observations can mitigate the misdistinguishing problem to some extent in this experiment. The reason that SeMAIL (No BoR) unexpectedly got an outstanding performance in the Walker task may be its task model remains more observation Information, which could benefit some different background situations. Although all the methods can not reach expert performance, our method still has a substantial advantage over other baselines.

\begin{table}[t]
\caption{Performance on tasks with non-overlap video backgrounds from expert observations, scaled by the expert and random agent returns.}
\label{tb:hete-results}
\vskip 0.15in
\begin{center}
\begin{small}
\setlength\tabcolsep{1pt}
\begin{sc}
\begin{tabular}{lcccr}
\toprule
Method & \fontsize{8.5pt}{\baselineskip}\selectfont{Walker Run}  & \fontsize{8.5pt}{\baselineskip}\selectfont{Cheetah Run} & \fontsize{8.5pt}{\baselineskip}\selectfont{Finger Spin}\\
\midrule
V-MAIL          & 0.16 $\pm$ 0.01 & 0.17 $\pm$ 0.02 & 0.02 $\pm$ 0.01 \\
DA-DAC          & 0.13 $\pm$ 0.02 & 0.19 $\pm$ 0.02 & 0.09 $\pm$ 0.03 \\
DisentanGAIL    & 0.27 $\pm$ 0.04 & 0.19 $\pm$ 0.03 &  \textbf{1.47 $\pm$ 0.07} \\
SeMAIL (No AcT)  & 0.10 $\pm$ 0.01 & 0.31 $\pm$ 0.04 & 0.07 $\pm$ 0.02 \\
SeMAIL (No BoR)  & \textbf{0.84 $\pm$ 0.01} & 0.38 $\pm$ 0.03 & 0.75 $\pm$ 0.08   \\
SeMAIL           & 0.38 $\pm$ 0.04 & \textbf{0.86 $\pm$ 0.04} & 0.82 $\pm$ 0.05 \\
\bottomrule
\end{tabular}
\end{sc}
\end{small}
\end{center}
\vskip -0.1in
\end{table}

\subsection{How well does the task-relevant space imagination benefit adversarial imitation learning?}
\label{exp: imagine}
To answer this question, we visualize the reconstruction of Cheetah Run's task information for expert observations and agent imaginations by SeMAIL and V-MAIL in \cref{fig:imagine}. The agent and the expert are trained in environments with the same task but with non-overlap video backgrounds. In expert observations, SeMAIL only extracts the cheetah's body as task information, while V-MAIL reconstructs all the information from the original image. In agent observations, SeMAIL imagines its behavior in the task-relevant latent space $\mathcal{Z}^+$, while V-MAIL imagines in the latent space containing all the information. Based on the reconstruction of the imagined data, we can see SeMAIL filter out irrelevant information while V-MAIL largely retains it. To evaluate the learning ability in the adversarial imitation learning process, we use the probability of whether the discriminator can distinguish behaviors from the agent or the expert. For agent imagination using V-MAIL, the discriminator gives a near-zero probability value since it carries all observed information which leads to distinguish from the expert (as the example mentioned in \cref{fig:motivation}). The agent of V-MAIL nearly can not receive positive training signals from the discriminator and underperforms in the task as a result. The discriminator of SeMAIL scores a probability close to 0.5 based on task-relevant information, which guides the agent's policy learning toward the expert. 


\section{Conclusion}
In model-based imitation learning, irrelevant information will cause the discriminator to have deceptive biases on the reward. We introduce the Action-conditioned Transition assumption to model task-relevant and distractor dynamics separately, and we name this new approach Separated Model-based Adversarial Imitation Learning (SeMAIL). SeMAIL extracts task-correlated features from the environment, which helps mitigate the problem of current MBIL methods struggling to learn in tasks with complex distractors. We provide a theoretical demonstration that the performance gap between the expert and the agent can be upper-bounded in task-relevant space. We verify the performance of SeMAIL on six frequently-used continuous control tasks, with complex distractors in observations. Further, we design ablation studies and visualize the adversarial imitation learning process to display the contribution of the key components. We conclude that SeMAIL can provide a relatively clean task-relevant latent space for the adversarial imitation learning process and largely improve performance on VIL tasks with complex observations. 

In our experiments, we claim that irrelevant information in most environments is a distraction to the agent's action decision and try to eliminate it. In particular cases, the irrelevant distractor may become task-relevant in a certain condition and influence the agent's decision-making. In this case, we can add a short period of the particular distractor states in the policy function and train the agent. We will consider this improvement in future work. Another further improvement is training the task-irrelevant models of the expert and the agent separately, which may benefit the scenario where distractors in their observations are majorly different.

\section*{Acknowledgements}
We would like to thank Han-Jia Ye, Qi-Wei Wang, Shaowei Zhang, Ziyuan Chen, and Jiayi Wu for valuable discussions. This work was supported by the National Key R$\And$D Program of China (2022ZD0114805).

\nocite{langley00}

\bibliography{my_paper}
\bibliographystyle{icml2023}

\newpage
\appendix
\onecolumn
\section{Proof}


\subsection{Proof of \cref{thm:gail bound}}\label{app:theorem}
\begin{theorem}
Consider a POMDP $\mathcal{M}$ with high-dimensional inputs such as images. Let $s_t$ be the ground-truth state and $z_t$ be the latent representation of the whole observation $o_t$. $z_t^+$ and $z_t^-$ represent for the task-relevant and irrelevant component, respectively, which meets the condition of $p(s_t|z_t, a_t) = p(s_t|z_t^+, a_t)p(s_t|z_t^-)$ based on the AcT assumption mentioned in \cref{rssm2}. Then, we have
\begin{equation}
    \mathbb{D}_f(\rho^{\pi}_\mathcal{M}(o, a)||\rho^{E}_\mathcal{M}(o, a)) \le 
     \mathbb{D}_f(\rho_{\mathcal{M}}^{\pi}(s, a)||\rho_{\mathcal{M}}^E(s, a))
     \le \mathbb{D}_f(\rho_{\mathcal{M}}^{\pi}(z^+, a)||\rho_{\mathcal{M}}^E(z^+, a))
\end{equation}
where $\mathbb{D}_f$ is a generic $f$-divergence.
\end{theorem}

\begin{proof} 
We get the result that $\mathbb{D}_f(\rho^{\pi}_\mathcal{M}(o, a)||\rho^{E}_\mathcal{M}(o, a)) \le \mathbb{D}_f(\rho^{\pi}_\mathcal{M}(s, a)||\rho^{E}_\mathcal{M}(s, a))\le \mathbb{D}_f(\rho^{\pi}_\mathcal{M}(z, a)||\rho^{E}_\mathcal{M}(z, a))$ from Theorem 1 which is proved in Rafael Rafailov's work \yrcite{rafailov2021visual}. Next, we need to prove $\mathbb{D}_f(\rho^{\pi}_\mathcal{M}(z, a)||\rho^{E}_\mathcal{M}(z, a)) \le \mathbb{D}_f(\rho^{\pi}_\mathcal{M}(z^+, a)||\rho^{E}_\mathcal{M}(z^+, a))$, which means that the gap between the agent and the expert occupancy measures upper-bounded in the latent state space $\mathcal{Z}$ can also be upper-bounded in the task-relevant state space $\mathcal{Z}^+$. 

\begin{align}
\mathbb{D}_f(\rho_\mathcal{M}^\pi(z,a)||\rho_\mathcal{M}^E(z,a)) &= \mathbb{E}_{z,a\sim\rho_{\mathcal{M}}^{E}(z,a)}\bigg[f\bigg(\frac{\rho_{\mathcal{M}}^{\pi}(z,a)}{\rho_{\mathcal{M}}^{E}(z,a)}\bigg)\bigg)\bigg]\\
&=\mathbb{E}_{z^+,z^-,a\sim\rho_{\mathcal{M}}^{E}(z,a)}\bigg[f\bigg(\frac{\rho_{\mathcal{M}}^{\pi}(z^+,a)P(z^{-}|z^+)}{\rho_{\mathcal{M}}^{E}(z^+,a)P(z^{-}|z^+)}\bigg)\bigg)\bigg]\\
&=\mathbb{E}_{z^+,z^-,a\sim\rho_{\mathcal{M}}^{E}(z,a)}\bigg[f\bigg(\frac{\rho_{\mathcal{M}}^{\pi}(z^+,a)}{\rho_{\mathcal{M}}^{E}(z^+,a)}\bigg)\bigg)\bigg]\\
&\le\mathbb{E}_{z^+,a\sim\rho_{\mathcal{M}}^{E}(z^+,a)}\bigg[f\bigg(\frac{\rho_{\mathcal{M}}^{\pi}(z^+,a)}{\rho_{\mathcal{M}}^{E}(z^+,a)}\bigg)\bigg)\bigg]\\
&=\mathbb{D}_f(\rho_\mathcal{M}^\pi(z^+,a)||\rho_\mathcal{M}^E(z^+,a))
\end{align}

The equality (10) follows the fact that $\rho_{\mathcal{M}}(z,a) =\rho_M(z^+,z^-,a)= \rho_{\mathcal{M}}(z^+,a)P(z^{-}|z^+,a)=\rho_{\mathcal{M}}(z^+,a)P(z^{-}|z^+)$, which means the task-relevant and -irrelevant parts can be partitioned, as mentioned in \cref{rssm2}. The inequality (12) is deflated with the help of our AcT assumption, which assumes that the forward dynamics in $\mathcal{Z}$ can be decoupled into two independent forward dynamics in $\mathcal{Z}^+$ and $\mathcal{Z}^-$.

\end{proof}

\section{Derivations}\label{app:deriv}
Previous model-based RL studies \cite{hafner2019dream,DBLP:conf/iclr/HafnerL0B21} use the information bottleneck objective \cite{DBLP:conf/nips/TishbyS00} to encourage model states to predict observations and rewards while limiting the capacity of information that contained in states. We remove the reward prediction in the objective for model-based imitation learning as follows:
\begin{equation}
    \max \text{I}(o_{1:T}, z_{1:T}|a_{1:T}) - \beta \text{I}(i_{1:T}, z_{1:T}|a_{1:T})
\end{equation}
Here, $i_t$ is indices of the dataset such that $p(o_t|i_t) = \delta(o_t-o_t^{\prime})$. The first term can be simplified using the mutual information definition and the non-negativity of 
the KL-divergence.

\begin{align}
    \text{I}(o_{1:T},z_{1:T}|a_{1:T})
    =~& \mathbb{E}_{p(o_{1:T}, z_{1:T}, a_{1:T})}\left[\ln p(o_{1:T}|z_{1:T}, a_{1:T})-\ln p(o_{1:T}|a_{1:T})\right]\label{eq:reduce_const}\\
    \overset{+}{=}~& \mathbb{E}_{p(o_{1:T}, z_{1:T}, a_{1:T})}\left[\ln p(o_{1:T}|z_{1:T}, a_{1:T})\right] \\
    \ge~& \mathbb{E}_{p(o_{1:T}, z_{1:T}, a_{1:T})}\left[\ln p(o_{1:T}|z_{1:T}, a_{1:T})\right] - \mathbb{D}_{\text{KL}}\Big(p(o_{1:T}|z_{1:T}, a_{1:T})||\prod_{t=1}^Tq(o_t|z_t)\Big)\\
    =~& \mathbb{E}_{q(z_{1:T}|o_{1:T}, a_{1:T})}[\sum_{t=1}^T\ln q(o_t|z_t)]\\
    =~& \sum\limits_{t=1}^{T}\left[\mathbb{E}_{q(z_{t}^+|o_{1:t},a_{1:t-1})q(z_{t}^-|o_{1:t})}\ln q(o_{t}|z_{t}^+,z_{t}^-)\right]\label{eq:recon}
\end{align}
The second term in the equality (\ref{eq:reduce_const}) can be regarded as constant for the observed data. The equality (\ref{eq:recon}) is obtained by the fact that the representation of observation $z_t$ can be decoupled into the task-relevant part $z^+_t$ and irrelevant part $z^-_t$.

For the second term, we obtain the upper bound of it with the non-negativity of the KL-divergence and the AcT assumption below:

\begin{align}
\text{I}(z_{1:T},i_{1:T}|a_{1:T}) =~ & \mathbb{E}_{p(o_{1:T}, z_{1:T}, a_{1:T},i_{1:T})}\left[\sum_{t=1}^T\ln p(z_t|z_{t-1},a_{t-1},i_t)-\ln p(z_t|z_{t-1}, a_{t-1})\right]\\
\le~&\mathbb{E}_{q(z_{1:T}|o_{1:T},a_{1:T-1})}\left[\sum_{t=1}^{T}\ln\frac{q(z_{t}|o_{t},z_{t-1},a_{t-1})}{p(z_{t}|a_{t-1},z_{t-1})}\right]\\
=~&\mathbb{E}_{q(z_{1:t-1}|o_{1:t-1},a_{1:t-2})}\left[\sum\limits_{t=1}^{T}\mathbb{E}_{q(z_{t}|o_{t},z_{t-1},a_{t-1})}\ln\frac{q(z_{t}|o_{t},z_{t-1},a_{t-1})}{p(z_{t}|a_{t-1},z_{t-1})}\right]\\
=~&\mathbb{E}_{q(z_{1:t-1}|o_{1:t-1},a_{1:t-2})}\left[\sum\limits_{t=1}^{T}\mathbb{E}_{q(z_{t}|o_{t},z_{t-1},a_{t-1})}\ln\frac{q(z_{t}^+|o_{t},z_{t-1}^+,a_{t-1})q(z_{t}^-|o_{t},z_{t-1}^-)}{p(z_{t}^+|z_{t-1}^+,a_{t-1})p(z_{t}^-|z_{t-1}^-)}\right]\\
=~&\mathbb{E}_{q(z_{1:t-1}^+|o_{1:t-1},a_{1:t-2})}\left[\sum\limits_{t=1}^{T}\mathbb{E}_{q(z_{t}^+|o_{t},z_{t-1}^+,a_{t-1})}\ln\frac{q(z_{t}^+|o_{t},z_{t-1}^+,a_{t-1})}{p(z_{t}^+|z_{t-1}^+,a_{t-1})}\right] \nonumber \\ 
&+\mathbb{E}_{q(z_{1:t-1}^-|o_{1:t-1})}\left[\sum\limits_{t=1}^{T}\mathbb{E}_{q(z_{t}^-|o_{t},z_{t-1}^-)}\ln\frac{q(z_{t}^-|o_{t},z_{t-1}^-)}{p(z_{t}^-|z_{t-1}^-)}\right]\\
=~&\sum\limits_{t=1}^{T}\Big(\mathbb{E}_{q(z_{t-1}^+|o_{1:t-1},a_{1:t-2})}\left[\mathbb{D}_{\text{KL}}({q(z_{t}^+|o_{t},z_{t-1}^+,a_{t-1})}\|{p(z_{t}^+|z_{t-1}^+,a_{t-1})})\right] \nonumber \\
&+\mathbb{E}_{q(z_{t-1}^-|o_{1:t-1})}\left[\mathbb{D}_{\text{KL}}({q(z_{t}^-|o_{t},z_{t-1}^-)}\|{p(z_{t}^-|z_{t-1}^-)})\right]\Big)
\end{align}
In practice, we use two pairs of observation decoders $q_{\phi^+}$ and $q_{\phi^-}$, forward dynamics models $p_{\theta^+}$ and $p_{\theta^-}$, and variational posterior models $q_{\psi^+}$ and $ q_{\psi^-}$ for the task-relevant and irrelevant branches respectively. We obtain the final objective of the separated models to be optimized as follows:
\begin{equation}
\begin{split}
    \max_{\theta^{+},\theta^{-}, \psi^{+}, \psi^{-}\atop \phi^{+}, \phi^{-}, \tilde{\phi}} \quad\mathbb{E}_{(o_{\tau},a_{\tau})\sim\mathcal{B}_{\pi}\cup\mathcal{B}_E}&\Big[\sum_{t=1}^T\big(\mathbb{E}_{q(z_{t-1}^-|o_{1:t-1})}[-\mathbb{D}_{\text{KL}}\left(q_{\psi^-}(z^-_t|o_t, z^-_{t-1})||p_{\theta^-}(z^-_t|z^-_{t-1})\right)]\\
    +&\mathbb{E}_{q(z_{t-1}^+|o_{1:t-1},a_{1:t-2})}[-\mathbb{D}_{\text{KL}}\left(q_{\psi^+}(z^+_t|o_t, z^+_{t-1}, a_{t-1})||p_{\theta^+}(z^+_t|z^+_{t-1}, a_{t-1})\right)]\\
    +&\mathbb{E}_{q(z_{t}^+|o_{1:t},a_{1:t-1})q(z_{t}^-|o_{1:t})}[\ln q_{\phi^{+}, \phi^{-}}(o_{t}|z^+_{t}, z^-_{t}) +\lambda\ln q_{\tilde{\phi}}(o_{t}|z^-_{t})]\big)\Big]
\end{split}
\end{equation}
where we add the background-only reconstruction loss mentioned in \cref{observation reconstruction}.

\section{Implementation Details}
\label{app:details}
\subsection{Environments and Tasks}
\begin{figure*}[t]
\vskip 0.2in
\begin{center}
\centerline{\includegraphics[width=\linewidth]{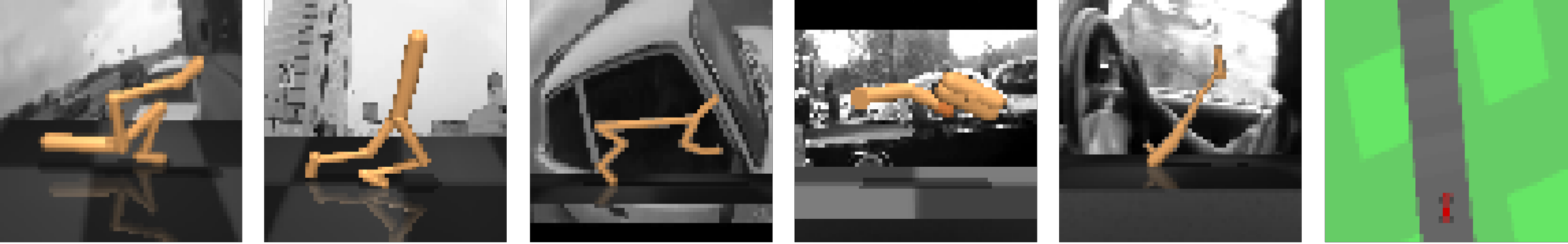}}
\caption{The environments used in our experiments: Walker Walk, Walker Run, Cheetah Run, Finger Spin, Hopper Hop, and Car Racing, where the first five replace backgrounds with natural videos except for Car Racing.}
\label{fig:env}
\end{center}
\vskip -0.2in
\end{figure*}

DeepMind Control Suite (DMC) \cite{DBLP:journals/corr/abs-1801-00690} is a set of reinforcement learning environments that includes a range of tasks, such as locomotion, manipulation, and navigation. We choose five locomotion tasks from DMC and replace the background wall with the grayscaled videos in the ``driving car'' class from the kinematic dataset \cite{DBLP:journals/corr/KayCSZHVVGBNSZ17}. These natural video backgrounds contain complex task-irrelevant information. The Car Racing environment from OpenAI Gym \cite{brockman2016openai} is a classic control task from pixels. The agent must learn to control the speed, steering, and other parameters to maximize its cumulative rewards. We cut out the bottom part of the frames to avoid the agent learning directly from the reward signals. The size of image observation in all environments is $64\times 64\times 3$. The example observations of these environments are shown in \cref{fig:env}.

In the first part of our experiments, we test our method SeMAIL and the baselines on all six environments. On the five environments of DMC, the agent and the expert observations background walls are replaced with the frames from the same videos. We design this task to verify SeMAIL's learning ability in environments with complex and noisy observations. We adopt the original image input for the Car Racing environment because it already contains much irrelevant information as time-varying blocks in observations. 

In the following experiment, we use Walker Run, Cheetah Run, and Finger Spin tasks to test SeMAIL's ability with backgrounds not appeared in expert observations. To achieve this, we design two non-overlap environments for each task. We train the expert in the first eight videos under the "driving car" class and collect the demonstration, then we train the agent in the last eight videos in the same class with the expert demonstration.

\subsection{Demonstration Data}
To make sure agent actions can be inferred from image inputs with distractors, we train an expert from scratch for each task and obtain the demonstration data. We train the expert with TIA \cite{DBLP:conf/icml/FuYAJ21} for locomotion tasks from DMC and train it with Dreamer \cite{hafner2019dream} for the Car Racing task. We prepare for all the methods with 10 expert demonstrations in each task.

\subsection{Pseudo Code}
The pseudo-code of our proposed SeMAIL is provided in \cref{alg:SeMAIL}.
\begin{algorithm}
   \caption{Training Procedure of SeMAIL}
   \label{alg:SeMAIL}
\begin{algorithmic}
   \STATE {\bfseries Input:} Policy replay buffer $\mathcal{B}_{\pi}$, Expert demonstrations $\mathcal{B}_{E}$\\
    \STATE Initialize forward dynamics model $p_{\theta^+}, p_{\theta^-}$, posterior encoder $q_{\psi^+}, q_{\psi^-}$, observation decoder $q_{\phi^+}, q_{\phi^-}, q_{\tilde{\phi}}$, policy $\pi$.
    \FOR{each time step $t=1\cdots T$}
        \STATE {// Rollout trajectories}
        \STATE {Infer the task-relevant latent state $z_t^+\sim q_{\psi^+}(\cdot| o_t, z_{t-1}^+, a_{t-1})$}
        \STATE {Sample action from policy $a_{t}\sim \pi(\cdot| z_t^+)$}
        \STATE {Execute action and get the next observation $o_{t+1}\leftarrow \text{env.step}(a_t)$}
    \ENDFOR
    \STATE {Add samples into the replay buffer $\mathcal{B}_{\pi}\leftarrow \mathcal{B}_{\pi}\cup \{(o_t, a_t)_{t=1}^T\}$}
    \FOR{training iteration $i=1\cdots \text{It}$}
        \STATE {// Learn separated models}
        \STATE {Sample minibatch $(o_{1:T}, a_{1:T-1})_{1:b}$ from the union buffer $\mathcal{B}_{\pi}\cup\mathcal{B}_E$}
        \STATE {Update the forward dynamics model $p_{\theta^+}, p_{\theta^-}$ and the posterior encoder $q_{\psi^+}, q_{\psi^-}$ with \cref{eq:rssm2}}
        \STATE {Update the observation decoder $q_{\phi^+}, q_{\phi^-}, q_{\tilde{\phi}}$ with \cref{eq:obz_rec}}
        \STATE {// Optimize policy}
        \STATE {Imagine the task-relevant latent states $z^{\pi+}_{1:H}$ by policy $\pi$ using the forward dynamics model $p_{\theta^+}$}
        \STATE {Sample expert's trajectories $(o_{1:T}^E, a_{1:T-1}^E)$ from demonstration buffer $\mathcal{B}_E$}
        \STATE {Infer the task-relevant latent states $z^{E+}_{1:T}$ using the posterior encoder $q_{\psi^+}$}
        \STATE {Train the discriminator on the data $(z^{\pi+}_{1:H-1}, a^{\pi}_{1:H-1}), (z^{E+}_{1:T-1}, a^E_{1:T-1})$ using \cref{eq:discrim}}
        \STATE {Update the policy $\pi$ to imitate the expert's behavior using \cref{eq:policy}}
    \ENDFOR
\end{algorithmic}
\end{algorithm}


\subsection{Networks and Hyperparameters}\label{app:detail net}
\textbf{Implementation.} We implement the proposed algorithm with TensorFlow 2 and run all the experiments on NVIDIA RTX 3090  for about 1000 GPU hours. We use the recurrent state space model \cite{DBLP:conf/icml/HafnerLFVHLD19} for the forward dynamics and the posterior encoder. The hidden sizes for the deterministic part and stochastic part are 200 and 30. We adopt the convolutional encoder and decoder used in TIA \cite{DBLP:conf/icml/FuYAJ21}. The size of all dense layers is 300, and the activation function is ELU. We use ADAM optimizer to train the network with batches of 64 sequences of length 50. The learning rate for the task and background model is 6e-5, and for the action net, value net, and discriminator is 8e-5. We clip gradient norms to 100 to stabilize the training process. To prevent training a too-strong discriminator, we add a gradient penalty term \cite{DBLP:conf/nips/GulrajaniAADC17} on the discriminator loss and set the value of weight as 1.0. The codes of SeMAIL are released in \href{https://github.com/yixiaoshenghua/SeMAIL}{https://github.com/yixiaoshenghua/SeMAIL}.

\textbf{Environment Hyperparameters.} For Walker Walk, Walker Run, Cheetah Run, Finger Spin, Hopper Hop, and Car Racing tasks, the values of background-only reconstruction $\lambda$ are 1.5, 0.25, 2, 1.5, 2, and 1, respectively. To encourage exploration, we add $\mathcal{N}(0, 0.3)$ noise on the output actions for locomotion tasks and $\mathcal{N}(0, 0.1)$ noise for Car Racing as used in \cite{rafailov2021visual}. The imagination horizon $H$ for locomotion tasks is 15, and for Car Racing is 10.

\textbf{Training Details.} To make this a fair comparison, all environments and network parameters are kept the same over SeMAIL and the compared algorithms. We initialize the dataset with 5 randomly collected episodes and train 100 iterations after collecting one episode in environments. We keep the action repeat times as 2 and set the discounting factor as 0.99 for all tasks. On five locomotion tasks, we adopt the official implementation for DisentanGAIL \cite{DBLP:conf/iclr/CetinC21}, which obtains pseudo rewards from the discriminator on the agent observations and learns policy based on the raw state observations.

\section{Extra Experimental Results}
\subsection{Quantitative Results}
\label{app:quanti_res}
In \cref{tb: homo-results}, we show the quantitative results of the experiments in \cref{exp: homo-source}. SeMAIL significantly outperforms the baselines in almost all tasks and achieves near-expert performance. Without the AcT assumption or removing the BoR loss, the performances of SeMAIL in locomotion tasks significantly decrease in either case. On the Car Racing task, the performance of SeMAIL (No BoR) outperforms SeMAIL, and V-MAIL achieves better performance than on other tasks. It is because the background of the car racing environment is relatively simple, and the task-relevant part takes up the large proportion. DisentanGAIL performs well on Finger Spin but fails on other tasks due to its agent actions change little in observations, which regularizing the mutual information between the expert and the agent data can help solve this task.

\begin{table}
\caption{Performance on six visual control tasks. We present the mean and std of final performance by running 10 trajectories over 4 seeds for SeMAIL and the baselines.}
\label{tb: homo-results}
\vskip 0.15in
\begin{center}
\begin{small}
\begin{sc}
\begin{tabular}{lccccccr}
\toprule
Method & Walker Walk & Walker Run & Cheetah Run & Finger Spin & Hopper Hop & Car Racing\\
\midrule
Expert          & 964.5 $\pm$ 14.7 & 539.5 $\pm$ 6.5 & 629.5 $\pm$ 44.6 & 332.0 $\pm$ 24.4 & 263.2 $\pm$ 7.4 & 936.3 $\pm$ 5.7 \\
V-MAIL          & 314.9 $\pm$ 295.6 & 155.1 $\pm$ 73.9 & 74.6 $\pm$ 48.9 & 3.3 $\pm$ 4.8 & 0.2 $\pm$ 1.0 & 537.8 $\pm$ 228.4 \\
DA-DAC          & 25.4 $\pm$ 13.3 & 27.7 $\pm$ 21.2 & 32.9 $\pm$ 17.0 & 0.1 $\pm$ 0.6 & 0.0 $\pm$ 0.1 & -57.9 $\pm$ 38.7  \\
DisentanGAIL    & 61.0 $\pm$ 6.6 & 68.7 $\pm$ 13.4 & 36.6 $\pm$ 17.3 & 133.2 $\pm$ 77.9 & 1.0 $\pm$ 1.1 & -62.0 $\pm$ 20.0  \\
SeMAIL (No AcT)  & 124.9 $\pm$ 14.1 & 100.0 $\pm$ 7.4 & 71.9 $\pm$ 4.4 & 30.5 $\pm$ 4.3  & 66.3 $\pm$ 8.5 &  24.1 $\pm$ 36.2  \\
SeMAIL (No BoR)  & 93.7 $\pm$ 15.5 & 337.5 $\pm$ 14.7 & 72.9 $\pm$ 5.8 & 23.4 $\pm$ 5.4 & 19.1 $\pm$ 2.1 & \textbf{942.9 $\pm$ 1.6} \\
SeMAIL    & \textbf{900.1 $\pm$ 57.6} & \textbf{463.4 $\pm$ 42.7} & \textbf{217.9 $\pm$ 125.8} & \textbf{161.1 $\pm$ 68.7} & \textbf{90.5 $\pm$ 75.6} & 901.2 $\pm$ 19.8\\
\bottomrule
\end{tabular}
\end{sc}
\end{small}
\end{center}
\vskip -0.1in
\end{table}

\subsection{Additional Visualization}
\label{app:add_visual}

We use the t-SNE plot \cite{JMLR:v9:vandermaaten08a} to visualize training result through the agent's and expert's observation embeddings of SeMAIL and V-MAIL in \cref{fig:tsne_vmail_SeMAIL}. We randomly select images from observations of the initial stage, the learned policy, and the expert, then generate the latent representation $z^+\sim q^+_\phi(z^+|o)$ from SeMAIL and $z\sim q(z|o)$ from V-MAIL. Since misguided by complex distractors, the embeddings cluster of V-MAIL shows a large gap between the trained agent and the expert. In SeMAIL, the sample embeddings almost cover the representation space of the expert observations, which indicates a well-trained result through our distractor-eliminating policy. It confirms that SeMAIL is qualified to solve real-world cases which may contain complex distractors.

\begin{figure*}[t]
\vskip 0.2in
\begin{center}
\centerline{\includegraphics[width=\linewidth]{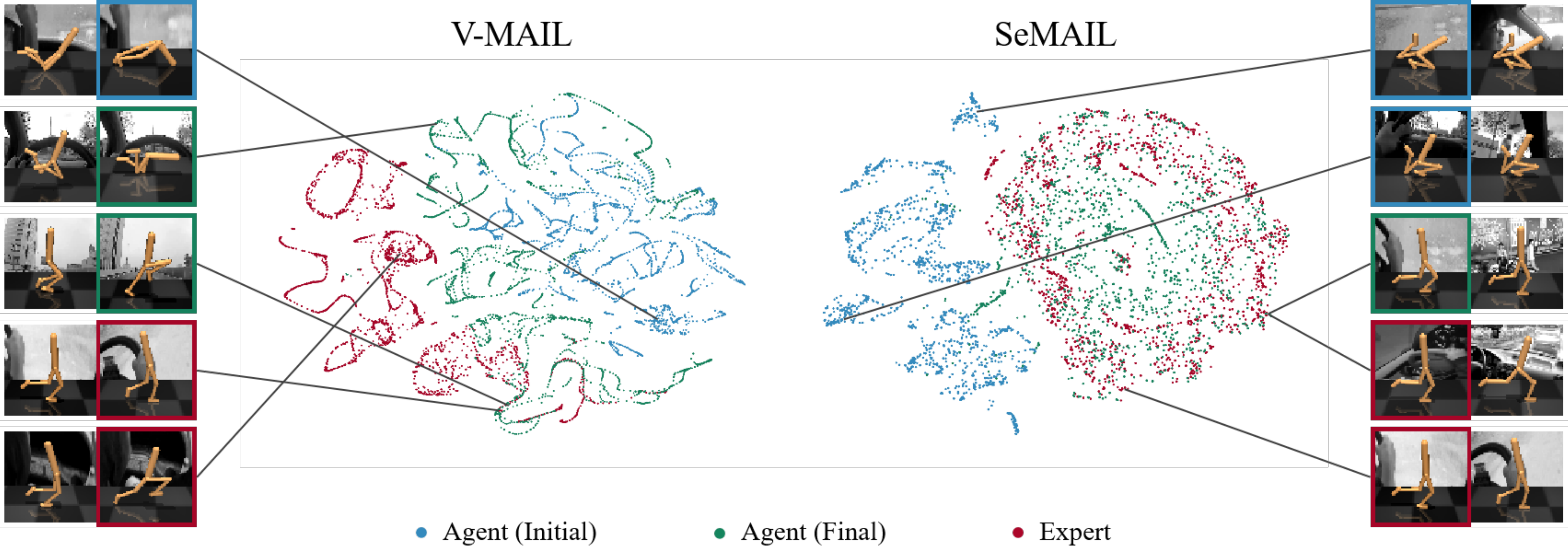}}
\caption{The t-SNE plot of data representations of the initial agent, the final learned agent, and the expert. Compared with the V-MAIL plot on the left, the embedding of our method SeMAIL shows that the learned agent behavior blends well with the expert, which indicates well-trained by eliminating distractors. }
\label{fig:tsne_vmail_SeMAIL}
\end{center}
\vskip -0.2in
\end{figure*}

\section{Additional Discussions}

\subsection{Pre-trained Visual Encoders}\label{app:add pretrain}
Equipping visual RL with a pre-trained encoder is an exciting direction like recent research \cite{DBLP:journals/corr/abs-2203-06173,DBLP:conf/icml/SeoLJA22}. In this paper, we are concerned about the practical problem in visual imitation learning (VIL) - how to extract task-relevant information from noisy observations without access to ground-truth rewards. To solve this, a good alignment between the agent and the expert behaviors (by separating the dynamics of task-relevant and -irrelevant parts) is a more important contributor than a good representation of observations (by pre-training or data augmentation). The compared method in our experiments, DA-DAC, only considers obtaining a good representation via augmenting the observed data, and its agent finally receives the wrong guidance from the expert and fails to complete the tasks (see the quantitative results in \cref{tb: homo-results}). We leave the pre-training approach combined with the SeMAIL framework for future exploration. 

\subsection{The Rationality of POMDP Assumption}\label{app:BMDP}

POMDP \cite{kaelbling:aij98} addresses that the agent can not directly access the ground truth states but only infer the belief states from observations. In previous RL work \cite{rafailov2021visual,DBLP:conf/iclr/BharadhwajBEL22} using the similar environments with us, they adopted POMDP as their basic environmental assumption. We follow their assumption because we aim to solve imitation learning tasks with noisy observations, which have some common ground with theirs. As detailed in \cref{app:detail net}, our transition model is built upon a recurrent state space model, which will generate belief hidden states from historical rollouts for policy action outputs. BMDP \cite{DBLP:conf/icml/DuKJAD019} highlights the block structure that each context determines its generating state uniquely, which can be regarded as an alternative formulation of our problem. We highly encourage researchers with an interest in BMDP to consult the original paper.

\subsection{The Practice of the AcT Assumption}\label{app:AcT} 
The AcT assumption claims that the forward dynamics $p(z_t|z_{t-1},a_{t-1})$ can be \textbf{independently} decomposed into the product of task-relevant dynamics $p(z_t^+ \mid z_{t-1}^{+}, a_{t-1})$ and -irrelevant dynamics $p(z_{t}^-|z_{t-1}^-)$, and each one is implemented by a stochastic probabilistic model. Recent studies, including Denoised MDPs \cite{DBLP:conf/icml/0001D0IZT22} and Iso-Dream \cite{NEURIPS2022_9316769a}, consider modeling the environmental dynamics using controllable information and pre-defined rewards. Denoised MDPs categorize information with controllablity and reward relativity. Iso-Dream posits that the latent states can be separated into controllable and noncontrollable part, while the agent makes decisions influenced by future noncontrollable states and subsequently affecting controllable states. Motivated by experimental observations and in light of the computational savings afforded by the independence hypothesis, we propose the AcT assumption. In contrast to previous work, the AcT assumption is stronger, positing that task-relevant and -irrelevant dynamics are independent. During the initial phase of training, due to the limited data collected by the agent, the separated models primarily rely on expert data for training. Because expert policies are optimal and focus only on task-relevant information to make actions unaffected by distractors, task-relevant and -irrelevant models can be considered independent at this stage. As training progress progresses and the agent’s policy approaches optimality, the agent becomes more resistant to irrelevant information interference, which supports the independence assumption in AcT. To testify the necessity of expert data for the AcT assumption in model learning of SeMAIL, we conduct an additional experiment in which separated models are trained without expert data while all other settings remained consistent. This method is denoted as SeMAIL (No Exp-M) in \cref{tb: act-results}. SeMAIL (No Exp-M) exhibits a significant performance decline, indicating the crucial role of expert data in model training for imitation learning under the AcT assumption.

In the implementation of SeMAIL based on the AcT assumption, we utilize a unimodal Gaussian distribution as the policy output. The environments and expert's data in our experiments all feature an optimal unimodal action distribution. This implicit environmental assumption which fits SeMAIL's policy output, may contribute to SeMAIL’s near-expert performance shown in \cref{tb: homo-results}. In complex environments with multimodal optimal behavior distributions, direct using current SeMAIL's policy may result in performance degradation. A potential solution is to modify the policy’s action output to a mixture of Gaussian distributions and introduce an additional loss, which aligns the policy’s action output with the expert’s action distribution without explicitly altering the AcT assumptions or other components of SeMAIL. 

In particular cases, a previous distractor could turn into a task-relevant factor in a certain condition and influence the agent’s further decision-making. In this case, we can add a short period of the particular distractor states in the policy function and train the agent. We will consider these extensions in our future work.

\begin{table}
\caption{Performance on five DMC tasks. We present the mean and std of 10 trajectories' return for SeMAIL and SeMAIL (No Exp-M).}
\label{tb: act-results}
\vskip 0.15in
\begin{center}
\begin{small}
\begin{sc}
\begin{tabular}{lccccc}
\toprule
Method & Walker Walk & Walker Run & Cheetah Run & Finger Spin & Hopper Hop \\
\midrule
SeMAIL    & 900.1 $\pm$ 57.6 & 463.4 $\pm$ 42.7 & 217.9 $\pm$ 125.8 & 161.1 $\pm$ 68.7 & 90.5 $\pm$ 75.6 \\
SeMAIL (No Exp-M) & 62.8 $\pm$ 11.5 & 63.7 $\pm$ 17.4 & 15.1 $\pm$ 13.3 & 3.0 $\pm$ 4.5 & 1.2 $\pm$ 2.5  \\
\bottomrule
\end{tabular}
\end{sc}
\end{small}
\end{center}
\vskip -0.1in
\end{table}

\subsection{Modifications of Non-IL Methods in Noisy Observations}

In \cref{sec:rl with noise}, we have clarified that the Non-IL methods in noisy observations heavily rely on using environmental rewards to exclude distractions. Our concern is how to learn a good policy under noisy observations (e.g., natural background videos) in imitation learning, which is free from reward-design issues. To our knowledge, our work is the first to consider this setting in IL. When designing the experiments, we attempted to modify related methods to adapt the setting of our work by utilizing the estimated pseudo rewards given by the discriminator to separate task-relevant and -irrelevant information. However, they failed to solve the tasks (the “with pseudo reward” row in \cref{tb: pseudoreward}). Using ill-estimated rewards can result in biased task-related representations and further mistaken estimation of rewards, which leads to a vicious cycle.

\begin{table}
\caption{Performance on three DMC tasks. We present the mean and std of 10 trajectories' return for random policy and Non-IL method with pseudo reward.}
\label{tb: pseudoreward}
\vskip 0.15in
\begin{center}
\begin{small}
\begin{sc}
\begin{tabular}{lccc}
\toprule
Method & Walker Run & Finger Spin & Hopper Hop \\
\midrule
Random Policy    & 29.0 $\pm$ 11.8 & 2.1 $\pm$ 2.8 & 0.0 $\pm$ 0.0 \\
with Pseudo Reward  & 50.4 $\pm$ 17.2 & 3.3 $\pm$ 5.0 & 0.7 $\pm$ 2.2  \\
\bottomrule
\end{tabular}
\end{sc}
\end{small}
\end{center}
\vskip -0.1in
\end{table}

\end{document}